%% file: main.tex
\documentclass[twoside]{article}

\usepackage{aistats2025}
\input{math_commands.tex}

\usepackage[round]{natbib}

\newtheorem{definition}{Definition}
\newtheorem{proposition}{Proposition}
\newtheorem{lemma}{Lemma}
\newtheorem{remark}{Remark}

\begin{document}

%

%

\twocolumn[

\aistatstitle{xPerT: Extended Persistence Transformer}

\aistatsauthor{ Sehun Kim }

\aistatsaddress{ Seoul National University } ]

\begin{abstract}
  A persistence diagram provides a compact summary of persistent homology, which captures the topological features of a space at different scales.  However, due to its nature as a set, incorporating it as a feature into a machine learning framework is challenging. Several methods have been proposed to use persistence diagrams as input for machine learning models, but they often require complex preprocessing steps and extensive hyperparameter tuning. In this paper, we propose a novel transformer architecture called the \textit{Extended Persistence Transformer (xPerT)}, which is highly scalable than the compared to Persformer, an existing transformer for persistence diagrams. xPerT reduces GPU memory usage by over 90\% and improves accuracy on multiple datasets. Additionally, xPerT does not require complex preprocessing steps or extensive hyperparameter tuning, making it easy to use in practice. Our code is available at \url{https://github.com/sehunfromdaegu/xpert}.
\end{abstract}

\section{Introduction} 
\label{sec:intro}

Topological Data Analysis (TDA) uses ideas from topology to explore the shape and structure of data, revealing patterns that traditional statistical methods may not fully grasp. TDA is not only an active area of mathematical research but also has broad practical applications across various fields. One of the key tools in TDA is \textit{persistent homology}, which captures the multi-scale topological features of a dataset. A summary of persistent homology is provided by the \textit{persistence diagram}, a multiset of points in the plane.

Persistent homology has been utilized in a wide range of areas, including biomolecules~\citep{biomolecule, biomolecule2}, material science~\citep{material, material2}, meteorology~\citep{whether}, and image analysis~\citep{image, image2}. However, the integration of TDA with machine learning models remains a challenge due to several reasons: (1) persistence diagram is an unordered set, which is not a natural input for machine learning models, (2) each persistence diagram has a different number of points, which complicates batch processing. Several methods have been proposed to address this issue by converting persistence diagrams into fixed-size feature vectors~\citep{persistenceimage,persistencelandscape,atol}. However, these methods often require extensive hyperparameter tuning and may fail to encompass the full range of topological information in the data. Another line of research explores using neural networks to directly process persistence diagrams, such as PersLay~\citep{perslay} and PLLay~\citep{pllay}. While these models have shown promising results, their application is limited due to complex hyperparameter choices and implementation difficulties.

The transformer architecture~\citep{attention} has recently emerged as a powerful model in many domains, including natural language processing, computer vision, audio etc. Persistence diagrams have also been incorporated into transformer models, as demonstrated by \textit{Persformer} (Reinauer et al., 2022). However, Persformer faces scalability issues, and its training process is often unstable in practice. In this work, we propose a novel transformer architecture for (extended) persistence diagrams called Extended Persistence Transformer (xPerT). The xPerT can directly process persistence diagrams, without requiring common preprocessing steps often needed in existing methods. The xPerT is highly scalable in terms of training time and GPU memory usage, and it does not require extensive hyperparameter tuning. We demonstrate the effectiveness of the xPerT model on classification tasks using two datasets: graph datasets and a dynamical system dataset.

\begin{figure}
  \centering
  \includegraphics[scale=0.15]{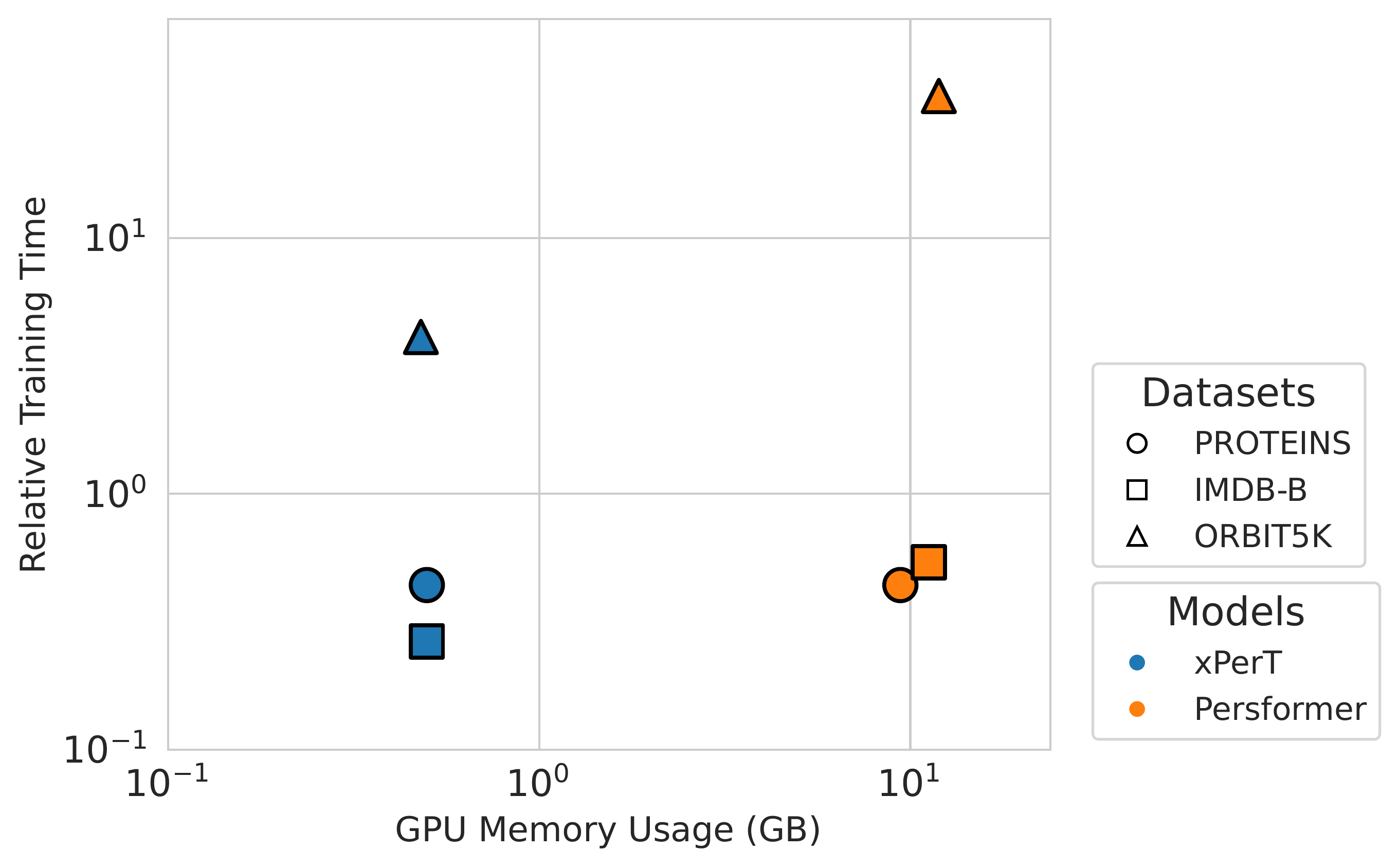}
  \caption{\textbf{Scaling.} Comparison of computational cost between persistence diagram transformer models in terms of training time and GPU memory usage (GB). The experiment was conducted using a batch size of 64 for the \texttt{PROTEINS} and \texttt{IMDB-B} datasets, and a batch size of 16 for the \texttt{ORBIT5K} dataset, as Persformer could not fit on our GPU (RTX 3090) with a batch size of 32.}
  \label{fig:scalability}
\end{figure}

\subsection{Related Works} \label{sec:related works}
The use of persistent homology in machine learning has been explored in various studies. One of the earliest approaches is the vectorization method, which converts a persistence diagram into a fixed-size feature vector. The \textit{persistence landscape} \citep{persistencelandscape} and \textit{persistence image} \citep{persistenceimage} are two popular vectorization methods. The persistence image also transforms a persistence diagram, but directly into a fixed-size vector by placing Gaussian kernels at each point and summing the values over a predefined grid. Both methods require the selection of numerous hyperparameters, which play a crucial role in determining the quality of the resulting vector, making the cross-validation process complex. \textit{ATOL} \citep{atol} offers an unsupervised vectorization approach by leveraging k-means clustering.

While most existing methods focus on vectorizing single-parameter persistent homology, extending these techniques to multi-parameter persistent homology is challenging due to the lack of a natural representation. However, some recent approaches have begun addressing these challenges, such as \textit{GRIL} \citep{gril} and \textit{HSM-MP-SW} \citep{multiparameter}.

Unlike vectorization methods, neural network-based approaches utilize persistence diagrams more directly. \textit{PersLay} \citep{perslay} maps a persistence diagram to a real number by:
\[
\text{PersLay}(D) \defeq \textit{op} \left(\left\{ w(p) \cdot \phi(p) \right\}_{p \in D}\right),
\]
where \(D\) denotes a persistence diagram, \(w: \sR^2 \to \sR^{2}\) is a learnable weight function, \smash{\(\phi: \sR^2 \to \sR^d\)} is a point transformation, and \textit{op} is a permutation-invariant operator. 
While PersLay processes persistence diagrams directly, \textit{PLLay} \citep{pllay} first computes the persistence landscape and then applies a neural network to the resulting feature vector.

\begin{figure}[t]
  \centering
  \includegraphics[scale=0.16]{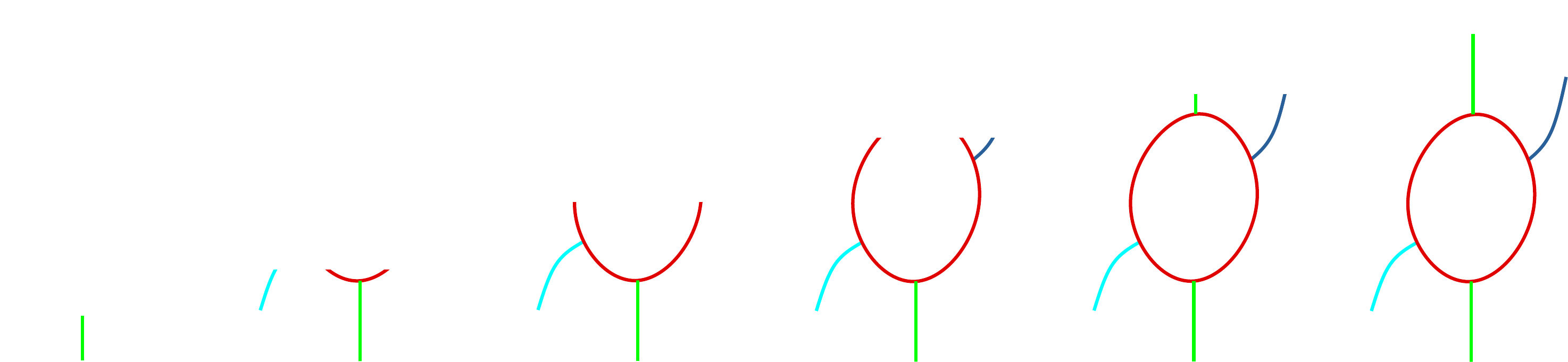}
  \caption{\textbf{Sublevel Set Filtration.} Six sublevel sets of the height function are shown. As \(c\) increases, the topology of \(X_{c}\) changes, which is represented by the ordinary persistence diagram. However, ordinary persistence cannot detect the appearance of the blue upright arm, which can instead be reflected by the superlevel set filtration \((X^{c})_{c \in \mathbb{R}}\) as \(c\) decreases.}\label{fig:filtration}
\end{figure}

More recently, \cite{persformer} introduced \textit{Persformer}, a model that applies a transformer architecture to persistence diagrams. In this approach, each point in the diagram is treated as a token, and the transformer operates without positional encodings. In batch processing, Persformer pads dummy points during preprocessing to ensure uniformity in the number of points across diagrams. Though simple, this method involves processing numerous tokens, leading to high computational costs and substantial GPU memory usage.

\paragraph{Contributions}
In this paper, we introduce the Extended Persistence Transformer (xPerT), a novel transformer architecture tailored for handling (extended) persistence diagrams in topological data analysis. Our main contributions are:
\begin{itemize}
   \item \textbf{Persistence Diagram Transformer}: We propose xPerT, which bridges the gap between persistence diagrams and transformer models by discretizing the diagrams into pixelized representations suitable for tokenization and input into the transformer architecture.
   \item \textbf{Scalability through Sparsity}: By leveraging the inherent sparsity of persistence diagrams, xPerT achieves high scalability in training time and GPU memory usage, making it efficient for large-scale applications.
   \item \textbf{Practical Implementation}: Our method is easy to implement and requires minimal hyperparameter tuning, lowering the barrier for practitioners and facilitating quick adoption.
\end{itemize}

\begin{figure*}[]
  \centering
  \includegraphics[scale=0.30]{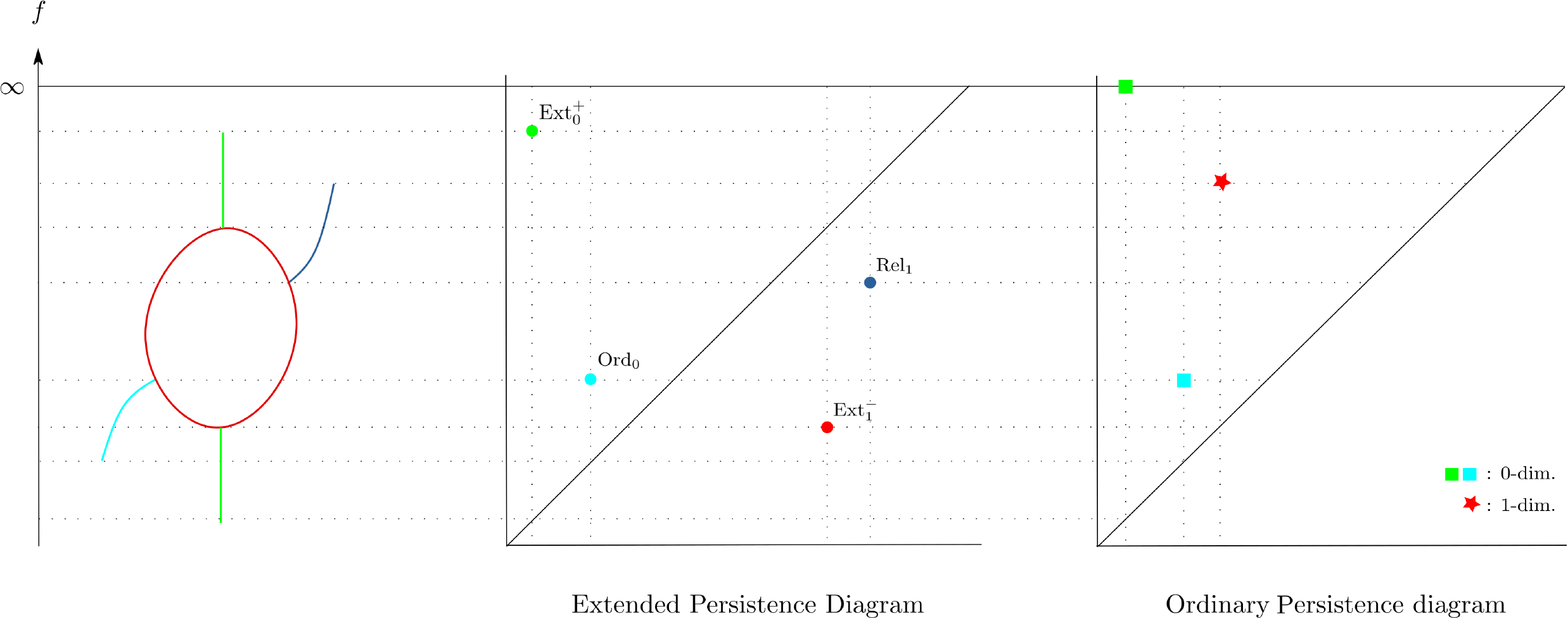}
  \caption{\textbf{Persistence Diagram.} The extended persistence diagram includes information about the blue upright arm and the maximum value of \(f\), which are not present in the ordinary persistence diagram. 
  (Left) A topological space equipped with a height function \(f\). (Center) Extended persistence diagram. (Right) Ordinary persistence diagram.}

  \label{fig:extendedpersistence}
\end{figure*}

\section{Background}\label{sec:background}
\subsection{Persistent Homology}

Persistent homology studies the evolution of a space, capturing its topological features at different scales. This section briefly introduces the fundamental concepts of persistent homology, and see \ref{sec:persistent homology} for more information. For readers unfamiliar with persistent homology, we recommend \cite{computationaltopology} for a comprehensive treatment.

Let \(f: X \rightarrow \sR\) be a continuous function from a topological space \(X\). The \textit{c-sublevel set} of \(f\), defined as \( X_{c} = \{x \in X : f(x) \leq c\} \) for \( c \in \mathbb{R} \), is an essential object for the understanding the topology of \( X \). These sublevel sets \((X_{c})_{c \in \sR}\) form an increasing sequence of spaces, known as a \textit{sublevel set filtration} as described in Figure~\ref{fig:filtration}. In particular, sublevel sets are central in \textit{Morse theory}, which analyzes the topology of \(X\) by studying the function \(f\).

A sequence of homology groups \(\{ H_{k}(X_{c}) \}_{c \in \sR}\) build upon the filtration \((X_{c})_{c \in sR}\) is called a \smash{\textit{k-dimensional persistent homology}.} Note that there is a natural map \(H_{k}(X_{c}) \rightarrow H_{k}(X_{c'})\) induced by the inclusion \(X_{c} \hookrightarrow X_{c'}\) for \(c \leq c'\). The evolution of topological features through the maps \(H_{k}(X_{c}) \rightarrow H_{k}(X_{c'})\) is encoded in the \textit{persistence diagram}, which is a multiset of points in the extended plane \(\sR \times (\sR \cup \{\infty\})\). If a topological feature appears at \(X_{b}\) and disappears at \(X_{d} \) for \(b < d\), the point \( (b,d) \) is included in the persistence diagram. If a topological feature appears at \(X_{b} \) and persists indefinitely, the point \( (b,\infty) \) is added. See the supplementary material in \ref{sec:persistent homology} for more details.

However, standard persistence diagrams provide only limited topological information, as they focus exclusively on the sublevel sets. To address this, \textit{extended persistence} was introduced in \cite{extendedpersistence}, incorporating additional information by utilizing the \textit{c-superlevel set} \( X^{c} = \{x \in X : f(x) \geq c\} \). An extended persistence diagram \(E\) consists of four components: \(\text{Ord}_{0}\), \(\text{Rel}_{1}\), \(\text{Ext}_{0}^{+}\), and \(\text{Ext}_{1}^{-}\), capturing more topological information than ordinary persistence. As shown in Figure \ref{fig:extendedpersistence}, the extended persistence diagram does not contain points at infinity, simplifying its use in machine learning models as well.

\subsection{Wasserstein Distance}
The collection of (ordinary) persistence diagrams forms a metric space with the \textit{Wasserstein distance}, which provides a measure of dissimilarity between two persistence diagrams. The Wasserstein distance is defined to be the infimum of the cost between all possible matchings \(\gamma\) between two persistence diagrams.

\begin{definition}[Wasserstein Distance]
   Given a persistence diagram \(D\), let \(\text{aug}(D)\) be the union of \(D\) and all points in the diagonal \(\Delta = \{(x,x):x \in \sR \}\) with infinite multiplicity. For two persistence diagrams \(D\) and \(D'\), the \((p,q)\)-Wasserstein distance between \(D\) and \(D'\) is defined as
   \[
   W_{p,q}(D, D^{'}) = \inf_{\gamma} \left( \sum_{u \in \text{aug}(D)} \lVert u - \gamma(u) \rVert_{q}^{p} \right)^{1/p},
   \]
   where \(\gamma: \text{aug}(D) \rightarrow \text{aug}(D^{'})\) ranges over all bijections.
\end{definition}

Intuitively, the Wasserstein distance measures how much `work' is required to match the points in one diagram to those in another, accounting for both the distance between points and the number of points involved. Throughout this paper, we will use the \((1,2)\)-Wasserstein distance \(W=W_{1,2}\).

\subsection{Heat Kernel Signature} \label{sec:hks}
The \textit{Heat Kernel Signature (HKS)} is a feature descriptor that encodes the intrinsic geometric properties of a shape. Originally defined for Riemannian manifolds~\citep{hks}, the discrete version of HKS can also be applied to graphs, allowing us to analyze their structural characteristics. We follow the approach in \cite{perslay}, using HKS values to generate extended persistence diagrams.

To define HKS, we first introduce the \textit{graph Laplacian}, a fundamental tool in graph analysis.

\begin{definition}[Graph Laplacian]
   Let \(G = (V,E)\) be an undirected graph with \(n\) vertices. The \textit{adjacency matrix} \(A\) of \(G\) is the \(n \times n\) matrix defined as
   \[
   A_{ij} = \begin{cases}
      1 & \text{if } (i,j) \in E, \\
      0 & \text{otherwise}.
   \end{cases}
   \]
   The \textit{normalized Laplacian matrix} of \(G\) is given by
   \[
   L = I - D^{-1/2} A D^{-1/2},
   \]
   where \(I\) is the identity matrix, and \(D\) is the diagonal matrix with \(D_{ii} = \sum_{j=1}^{n} A_{ij}\).
\end{definition}

\begin{definition}[Heat Kernel Signature]
  Given a graph \(G\) with a diffusion parameter \(t > 0\), the Heat Kernel Signature is the function \(H_{t} : V \rightarrow \mathbb{R}\) defined at each node \(v \in V\) by
  \[
  H_{t}(v) = \sum_{i=1}^{n} \exp(-\lambda_{i} t) \langle \phi_{i}, v \rangle^{2},
  \]
  where \(\lambda_{i}\) are the eigenvalues and \(\langle \phi_{i}, v \rangle\) is the value of the \(i\)-th eigenvector at node \(v\).
\end{definition}

For a fixed diffusion parameter \(t > 0\), \(H_t\) assigns a real number to each node of the graph, encoding the intrinsic geometric properties of the graph. The diffusion parameter \(t\) controls the scale at which the graph's geometric features are captured, with smaller values of \(t\) focusing on local structures and larger values capturing global properties. By assigning a real number to each node, \(H_t\) encodes essential structural information of the graph.

\section{Pixelized Persistence Diagram}

In this section, we introduce the \textit{Pixelized Persistence Diagram (PPD)}, an efficient representation of persistence diagrams designed for transformer inputs. The PPD is constructed by first applying instance normalization to handle varying scales and then projecting the normalized diagram onto a discrete grid. This approach ensures that diagrams of different scales are represented consistently, facilitating their use in machine learning models while retaining the stability properties of the original diagrams. The process is straightforward and can be implemented in just a few lines of code.

The PPD shares similarities with the Persistence Image (PI) but offers specific advantages when integrated with transformer architectures: (1) PPD requires less or no hyperparameter tuning, and (2) more importantly, PPD contains many zero-value pixels, allowing only non-zero pixels to be utilized. This sparsity significantly reduces the number of tokens, improving computational efficiency and scalability.

\subsection{Projection of Persistence Diagrams}
Given a rotated persistence diagram \( D_r \), we aim to project it onto a discrete grid to create a pixelized representation. For a fixed grid size \(\delta > 0\), we discretize the birth-persistence plane into grid cells:
\[
I_{k,l} = [k \delta, (k + 1) \delta) \times [l \delta, (l + 1) \delta), \quad k, l \in \mathbb{N}.
\]
Each point \( (b, p) \in D_r \) is associated with the grid cell \( I_{k,l} \) containing it.

\begin{definition}[Projection of Persistence Diagram]
The projection map \( \Pi_\delta : \sR^{2}_{\geq 0} \rightarrow \sR^{2}_{\geq 0}  \) is defined by mapping each point \( (b, p) \in D_r \) to the center of the grid cell containing it:
\[
\Pi_\delta(b, p) = \left( \left( k + \tfrac{1}{2} \right) \delta, \left( l + \tfrac{1}{2} \right) \delta \right), \quad \text{where } (b, p) \in I_{k,l}.
\]
The \textit{projected persistence diagram} is then \[ \Pi_\delta(D_r) = \{ \Pi_\delta(b, p) : (b, p) \in D_r \}. \]
\end{definition}

\begin{figure}[t]
  \centering
  \includegraphics[scale=0.4]{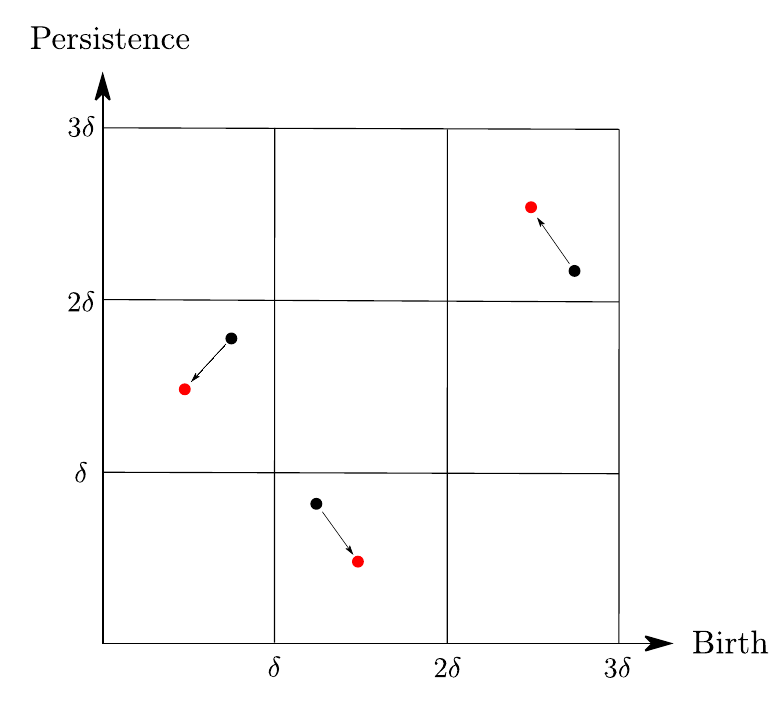}
  \caption{\textbf{Projection of Persistence Diagram.} Each point in \( D_r \) is projected onto the center of its corresponding grid cell using the projection map \( \Pi_\delta \). }
  \end{figure}
  
  \paragraph{Stability of the Projected Persistence Diagram}
  
  The projection operation \( \Pi_\delta \) applied to any rotated persistence diagram preserves stability with respect to the Wasserstein distance.
  
  \begin{proposition}\label{prop:stability}
  Let \( D \) and \( D' \) be two persistence diagrams, and let \( \delta < W(D, D') \). Then we have:
  \begin{align*}
    W\Big( \Pi_\delta(D_r), & \Pi_\delta(D'_r) \Big) \\
    &\leq \left( \frac{\sqrt{|D|}+\sqrt{|D'|}}{\sqrt{2}} + \sqrt{3} \right) W(D, D').
  \end{align*}
  
  \end{proposition}
  \begin{proof}
    See Appendix~\ref{proof:stability}.
    \end{proof}

  \subsection{Pixelized Persistence Diagram}

  The projected persistence diagram \( \Pi_\delta(D_r) \) is a digital-image-like representation, where each point corresponds to the center of a grid cell. However, the scale of \( D_r \) can vary between different persistence diagrams, leading to inconsistencies in the representation. To address this, we apply instance normalization before projecting the diagram.
  
  We define the \textit{instance-nomalized diagram} as
  \[
  \texttt{Norm}(D_r) = \left\{ \left( \frac{b}{b_{\text{max}}}, \frac{p}{p_{\text{max}}} \right) : (b, p) \in D_r \right\},
  \]
  where \(b_{\text{max}} = \max \{ b : (b, p) \in D_r \}\), \(p_{\text{max}} = \max \{ p : (b, p) \in D_r \}\) are the maximum values of birth, and persistence in \( D_r \), respectively. This maps the rotated diagram into the unit square.
  
  \begin{remark}
  If \( b_{\text{max}} = 0 \) (e.g., in 0-dimensional diagrams where all birth times are zero), we set \( b_{\text{max}} = 1 \) to avoid division by zero. Since all \( b \) values are zero, this normalization leaves them unchanged.
  \end{remark}
  \begin{figure*}[h]
    \centering
    \includegraphics[width=.9\textwidth]{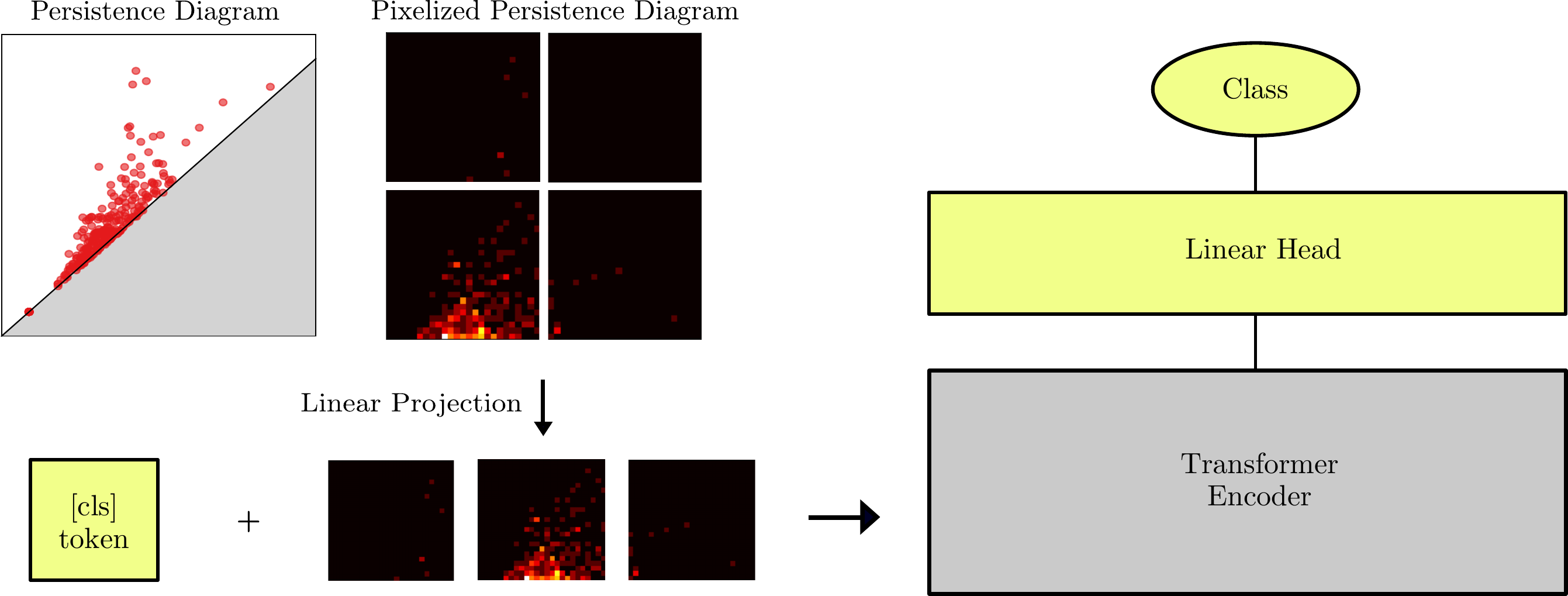}
    \caption{\textbf{xPerT Overview.} The persistence diagram is pixelized and split into fixed-size patches (left, top). These patches are linearly transformed into token vectors, with a \texttt{[cls]} token added (left, bottom). Empty patches are excluded from the transformer input.(Right) The token vectors are processed by the transformer model, and the output of the \texttt{[cls]} token is fed to the linear head for classification.}
    \label{fig:xpert}
  \end{figure*}
  
  The \textit{Pixelized Persistence Diagram (PPD)} is then the pixelized representation of \(\Pi_{\delta} \big(\texttt{Norm}(D_r)\big)\), with \(\delta = 1/H\). Concretely,  let \( H > 0 \) be a positive integer representing the grid resolution. We partition the unit square \( [0, 1] \times [0, 1] \) into \( H \times H \) grid cells:
  \[
  I_{i,j} = \left[ \frac{i - 1}{H}, \frac{i}{H} \right) \times \left[ \frac{j - 1}{H}, \frac{j}{H} \right), \quad i, j = 1, 2, \dots, H.
  \]

  \begin{definition}[Pixelized Persistence Diagram]
    The \textbf{Pixelized Persistence Diagram (PPD)} \smash{\( \ppd_H(D) \in \mathbb{N}^{H \times H} \)} is an integer matrix where each entry \( \ppd_H(D)_{i,j} \) represents the number of points from \( \texttt{Norm}(D_r) \) that fall into the \( (i,j) \)-th grid cell:
    \[
    \ppd_H(D)_{i,j} = \left| \left\{ (b, p) \in \texttt{Norm}(D_r) : \left( b, p \right) \in I_{i,j} \right\} \right|.
    \]
    \end{definition}
    
    \subsection{Extended Pixelized Persistence Diagram} \label{sec:ExtendedPPD}
    Having defined the PPD for a single diagram, we now extend the definition to the extended persistence diagram. The extended pixelized persistence diagram \(\ppd_{H}(E)\) of an extended persistence diagram \(E = \{\text{Ord}_{0}, \text{Rel}_{1}, \text{Ext}_{0}^{+}, \text{Ext}_{1}^{-} \}\) is obtained discretizing each diagram:
    
    \begin{enumerate}   
       \item Transpose the diagrams \(\text{Rel}_{1}\) and \(\text{Ext}_{1}^{-}\) so that all four diagrams are positioned in the upper half-plane, above the diagonal. 
       
       \item Rotate diagrams \(\{\text{Ord}_{0}, (\text{Rel}_{1})^{T}, \text{Ext}_{0}^{+}, (\text{Ext}_{1}^{-})^{T} \}\): \[\{R_{\text{ord}}, R_{\text{rel}}, {R_{\text{ext}^{+}}}, R_{\text{ext}^{-}}\}.\]
       
       \item Compute the PPD for each rotated diagram. The Extended Pixelized Persistence Diagram \(\ppd_{H}(E)\) is the set:
       \[
        \{ \ppd_{H}(R_{\text{ord}}), \ppd_{H}(R_{\text{rel}}), \ppd_{H}(R_{\text{ext}^{+}}), \ppd_{H}(R_{\text{ext}^{-}}) \}\footnote{   When computing the PPDs for the extended persistence diagram, the same \(b_{\text{max}}\) and \(p_{\text{max}}\) values are applied to all four diagrams to ensure consistency.}.
       \]
    \end{enumerate}


\section{xPerT}
We now describe the Extended Persistence Transformer (xPerT), illustrated in Figure \ref{fig:xpert}. The xPerT model processes a set of persistence diagrams by transforming them into sequences of tokens, which are then fed into a transformer model.  

\subsection{Tokenization} \label{sec:Tokenization}
To feed persistence diagrams into a transformer, they must first be converted into sequences of tokens. Here, we describe the tokenization process for extended persistence diagrams. Due to the unique nature of 0-dimensional persistence diagrams, where points typically lie along the \(y\)-axis, a slightly different tokenization method is used, as detailed in Section \ref{sec:Tokenization OPD} of the supplementary material.

Given an extended persistence diagram \(E\), we first discretize it into an extended PPD \(\ppd_{H}(E)\) (section \ref{sec:ExtendedPPD}), where \(H > 0\) is the resolution. Let \(\ppd \in \sN^{4 \times H \times H}\) be the tensor obtained by stacking individual PPD in \(\ppd_{H}(E)\) along the channel dimension. This multi-channel representation is then divided into \(N = (H/P)^2\) patches, where each patch is of size \(4 \times P \times P\), and \(P\) is the patch size. Each patch is then flattened into a vector in \(\sN^{4 P^2}\), resulting in the sequence \(\{\ppd_{\text{patch}}^{1}, \dots, \ppd_{\text{patch}}^{N}\}\).

Finally, we apply a linear transformation to each patch to obtain token embeddings:
\[
\text{Tokens} = \left\{\mathbf{E}\ppd_{\text{patch}}^1, \dots, \mathbf{E}\ppd_{\text{patch}}^N\right\}, \quad \mathbf{E}\ppd_{\text{patch}}^i \in \mathbb{R}^D,
\]
where \(\mathbf{E} \in \mathbb{R}^{D \times (4 P^2)}\) is a learnable projection matrix, and \(D\) is the embedding dimension. Note that this tokenization approach closely parallels the process used in the Vision Transformer \cite{vit}, where images are similarly divided into patches and transformed into token embeddings.

After tokenization, we prepend a classification \texttt{[cls]} token to the sequence, which aggregates information for downstream tasks. To retain spatial relationships, we add standard 2D sinusoidal positional encodings to the token embeddings.

\paragraph{Sparsity of Patches.}
One of the key advantages of xPerT is its ability to leverage the inherent sparsity in persistence diagrams. Points in persistence diagrams are often distributed non-uniformly, resulting in many patches being zero vectors. The xPerT model takes advantage of this sparsity by processing only the non-zero patches, which significantly reduces the number of tokens and computational costs.

\paragraph{Classification Head.}
We apply a single-layer linear head to the \texttt{[cls]} token, followed by a softmax layer to output the probability distribution over the classes.

\section{Experiments}

In this section, we present the experimental results of the xPerT model on graph and dynamical system classification tasks. We compare xPerT's performance with state-of-the-art methods related to persistent homology, including both single-parameter and multi-parameter approaches.

\paragraph{Architecture.}
We use the same xPerT architecture for all experiments to demonstrate that it performs well without extensive hyperparameter tuning.\footnote{Further hyperparameter optimization could improve performance, as shown in the ablation study.} The transformer model used in the experiments consists of 5 layers with 8 attention heads, with a token dimension set to 192. The resolution of the pixelized persistence diagram (PPD) is \(H=50\), with a patch size of \(P=5\), resulting in at most 100 patches per (extended) persistence diagram. However, due to the sparsity of the diagrams, the actual number of patches is often much smaller. For instance, the average number of non-zero patches in the \texttt{ORBIT5K} dataset is 8.4 for the 0-dimensional diagrams and 23.2 for the 1-dimensional diagrams.

\begin{table*}[]
  \centering
  \caption{\textbf{Graph classification.} The average classification accuracy across different datasets over 10-fold cross-validation. (Top) Methods leveraging single or multi-parameter persistent homology. (Bottom) Methods that combine persistent homology with Graph Isomorphism Network (GIN) models. The best results are highlighted in bold, and the second-best are underlined.}
  \resizebox{\textwidth}{!}{
     \begin{tabular}{lcccccc}\\
        \textbf{Method}  & \textbf{\texttt{IMDB-B}}          & \textbf{\texttt{IMDB-M}}              & \textbf{\texttt{MUTAG}}            & \textbf{\texttt{PROTEINS}}          & \textbf{\texttt{COX2}}             & \textbf{\texttt{DHFR}}   \\ \hline
        PersLay\(^{\dagger}\)  
                         & 71.2                     & 48.8                         & \underline{89.8}          & \underline{74.8}          & \underline{80.9}          & 80.3             \\
        ATOL             & 69.2 \(\pm\) 4.1         & 42.9 \(\pm\) 2.3             & 88.3 \(\pm\) 3.9          & 72.6 \(\pm\) 1.9          & 80.0 \(\pm\) 7.6          & 81.2 \(\pm\) 4.8 \\

        HSM-MP-SW\(^{\dagger}\)
                         &\textbf{74.7 \(\pm\) 5.0} & \underline{50.3 \(\pm\) 3.5} & 86.8 \(\pm\) 7.1          & 74.1 \(\pm\) 2.0          & 77.9 \(\pm\) 1.3          & \textbf{82.8 \(\pm\) 5.0} \\

        GRIL\(^{\dagger}\)
                       & 65.2 \(\pm\) 2.6           &     \(-\)\tablefootnote{GRIL was not evaluated on the \texttt{IMDB-MULTI} dataset.}                    & 87.8 \(\pm\) 4.2          & 70.9 \(\pm\) 3.1          & 79.8 \(\pm\) 2.9          & 77.6 \(\pm\) 2.5 \\

        Persformer       & 68.9 \(\pm\) 8.8         & \textbf{51.7 \(\pm\) 3.3}    & 89.4 \(\pm\) 4.0          & 72.0 \(\pm\) 6.7          & 78.2 \(\pm\)0.8 \tablefootnote{The model consistently predicts the majority class, leading to inflated accuracy due to class imbalance.}                     & 64.8 \(\pm\) 2.3 \\                         

        xPerT          &\underline{72.6 \(\pm\) 3.4}& 50.0 \(\pm\) 2.0             & \textbf{91.0 \(\pm\) 5.2} & \textbf{75.7 \(\pm\) 3.8} & \textbf{84.4 \(\pm\) 3.5} & \underline{81.9 \(\pm\) 2.9} \\ 
        \\[-1.6ex] \hline \\[-1.6ex]

        GIN            &\underline{75.3 \(\pm\) 4.8}& \textbf{52.3 \(\pm\) 3.6}   &\underline{94.1 \(\pm\) 3.8}& 76.6 \(\pm\) 3.0          & 85.4 \(\pm\) 3.4          & \textbf{84.5 \(\pm\) 3.6} \\  
        { }+ GRIL(sum)\(^{\dagger}\) 
                       & 74.2 \(\pm\) 2.8           &     \(-\)\footnote[3]                   & 89.3 \(\pm\) 4.8          & 71.9 \(\pm\) 3.2          & 79.2 \(\pm\) 4.9          & 78.5 \(\pm\) 5.8 \\
        { }+ xPerT(sum)& \textbf{76.1 \(\pm\) 2.7}  & 51.1 \(\pm\) 1.9             & 93.7 \(\pm\) 3.9         &\underline{77.6 \(\pm\) 4.7}& \textbf{85.9 \(\pm\) 2.9} & \underline{83.5 \(\pm\) 3.4} \\
        { }+ xPerT(cat)& 75.7 \(\pm\) 2.3           & \underline{52.1 \(\pm\) 2.0} & \textbf{94.7 \(\pm\) 5.3} & \textbf{78.4 \(\pm\) 5.1}&\underline{85.7 \(\pm\) 3.9}& 81.9 \(\pm\) 2.6 \\
        \end{tabular}
  }
  \label{tab:graph} 
\end{table*}

\subsection{Classification on Graph Datasets} \label{sec:graph classification}
Given a graph, we compute the heat kernel signature (HKS) on the graph with diffusion parameter \(t=1.0\), which is used to generate the extended persistence diagram. For detailed hyperparameters, see the supplementary material in \ref{sec:hyperparameters}.

\paragraph{Graph Datasets.} We evaluate xPerT on several widely used graph classification datasets: \texttt{IMDB-BINARY} , \texttt{IMDB-MULTI}, \texttt{MUTAG}, \texttt{PROTEINS}, \texttt{COX2}, and \texttt{DHFR}. The \texttt{IMDB} datasets consist of social network graphs, while the remaining datasets are derived from biological and medical domains (see \cite{tudataset} for details).


\paragraph{Baselines and Results.} 
Table \ref{tab:graph} summarizes the performance of xPerT compared with state-of-the-art methods related to persistent homology, as described in Section~\ref{sec:related works}. We compare xPerT with PersLay \citep{perslay}, ATOL \citep{atol}, and Persformer \citep{persformer}, which, as discussed earlier, use extended persistence diagrams generated from the HKS function. Additionally, we include comparisons with GRIL \citep{gril} and HSM-MP-SW \citep{multiparameter}, which employ multi-parameter persistent homology without relying on persistence diagrams. This allows us to evaluate xPerT's performance against both single-parameter and multi-parameter persistent homology approaches. Detailed settings for each model are provided in Table~\ref{tab:baselines} in the supplementary material.

In summary, xPerT demonstrates strong and consistent performance across various datasets, often surpassing or matching state-of-the-art methods. Furthermore, xPerT is flexible and can be seamlessly combined with models like GIN, making it adaptable for use with other deep learning architectures.

\subsection{Dynamical System Dataset.}
Table \ref{exp:orbit} shows the average classification accuracy of xPerT on the dynamical system datasets over 5 independent runs. We evaluate xPerT on two commonly used datasets in topological data analysis: \texttt{ORBIT5K} and \texttt{ORBIT100K}. These datasets consist of simulated orbits with distinct topological characteristics, generated using the recursive equations:
\[
\begin{aligned}
    x_{n+1} &= x_n + ry_n(1 - y_n) &\mod 1 \\
    y_{n+1} &= y_n + rx_{n+1}(1 - x_{n+1}) &\mod 1
\end{aligned}
\]
Each orbit is initialized with a random initial point \((x_0, y_0) \in [0, 1]^2\) and a parameter \smash{\(r \in \{2.5, 3.5, 4.0, 4.1, 4.3\}\)}. The task is to predict the parameter \(r\) that generated each orbit, given its 0-dimensional and 1-dimensional persistence diagrams generated by the \textit{weak alpha filtration} (Section~\ref{sec:Tokenization OPD}). 

The \texttt{ORBIT5K} dataset contains 5,000 point clouds, each with 1,000 points per orbit, while the larger \texttt{ORBIT100K} dataset consists of 100,000 point clouds, with 20,000 orbits for each \(r\) value. Both datasets are split into 70\% training and 30\% testing sets.


\begin{table}

  \centering
  \caption{\textbf{Orbit Classification.} Mean classification accuracy over 5 independent runs for the \texttt{ORBIT5K} and \texttt{ORBIT100K} datasets. Results marked with \(\dagger\) are taken from the original papers. \\} 
     \begin{tabular}{lcc}
        \textbf{Method} & \textbf{\texttt{ORBIT5K}} & \textbf{\texttt{ORBIT100K}} \\ \hline
        PersLay\(^{\dagger}\) & \underline{87.7 \(\pm\) 1.0}   & 89.2 \(\pm\) 0.3 \\
        ATOL         & 72.2 \(\pm\) 1.5 & 68.8 \(\pm\) 8.0 \\
        Persformer\(^{\dagger}\) & \textbf{91.2 \(\pm\) 0.8} & \textbf{92.0 \(\pm\) 0.4} \\
        Persformer\tablefootnote{Reproduced result. The model was not trainable in both \texttt{ORBIT5K} and \texttt{ORBIT100K}. We used the code from \href{https://github.com/giotto-ai/giotto-deep}{https://github.com/giotto-ai/giotto-deep}.}& 28.2 \(\pm\) 7.3 & \(-\) \\
        xPerT  & 87.0 \(\pm\) 0.7 & \underline{91.1 \(\pm\) 0.1} \\
        \end{tabular} \label{exp:orbit}

\end{table}
\begin{table*}[t]
  \centering
  \caption{\textbf{Ablating patch size (Graph).} Average accuracies over 10-fold cross-validation. The effect of patch size is less clear in graph datasets.}
  \resizebox{\textwidth}{!}{
     \begin{tabular}{lcccccc}\\
        \textbf{Patch Size}& \textbf{\texttt{IMDB-B}} & \textbf{\texttt{IMDB-M}} & \textbf{\texttt{MUTAG}}  & \textbf{\texttt{PROTEINS}}& \textbf{\texttt{COX2}}   & \textbf{\texttt{DHFR}}\\ \hline
        \(P = 2\)          & 71.9 \(\pm\) 3.8& 48.1 \(\pm\) 4.0& 89.4 \(\pm\) 5.7& \textbf{75.7 \(\pm\) 3.0}& 82.6 \(\pm\) 4.6& 80.7 \(\pm\) 2.4\\
        \(P = 5\)          & 72.6 \(\pm\) 3.4& \textbf{50.0 \(\pm\) 2.0}& \textbf{90.0 \(\pm\) 5.2}& \textbf{75.7 \(\pm\) 3.8}&\textbf{84.4 \(\pm\) 3.5}& \textbf{81.9 \(\pm\) 2.9}\\
        \(P = 10\)         & \textbf{74.5 \(\pm\) 4.1}& 49.2 \(\pm\) 3.3& 89.9 \(\pm\) 6.1& 75.5 \(\pm\) 3.4& 82.7 \(\pm\) 3.6& 81.6 \(\pm\) 4.2
        \end{tabular}\label{abl:patchsizegraph}
  }
\end{table*}
\section{Ablation Study}
\subsection{Patch Size.}

Tables~\ref{abl:patchsizeorbit} and~\ref{abl:patchsizegraph} show the effect of patch size on the dynamical system and graph datasets, respectively. Reducing the patch size increases the number of patches, providing a finer resolution for the pixelized persistence diagrams. We observe that the ORBIT5K dataset benefits significantly from smaller patch sizes, as they allow the model to extract more detailed topological information. However, the effect of decreasing the patch size is less pronounced on the ORBIT100K dataset, likely due to the increased scale of the dataset. 

In contrast, the effect of patch size on the graph datasets is less evident. This may be due to the fact that persistence diagrams generated from graph datasets tend to have much fewer points compared to those from the dynamical system datasets. As a result, reducing the patch size may not provide significant additional information.

\begin{table}[t]
  \centering
  \caption{\textbf{Ablating patch size (Orbit).} Impact of varying patch sizes on classification performance for the \texttt{ORBIT5K} and \texttt{ORBIT100K} datasets. The results show that xPerT benefits from smaller patch sizes, with mean accuracy reported over 5 independent runs. \\} 
     \begin{tabular}{lcc}
        \textbf{Patch Size} & \textbf{\texttt{ORBIT5K}}            & \textbf{\texttt{ORBIT100K}} \\ \hline
        \(P = 2\)           & \textbf{88.1 \(\pm\) 0.4}   & \textbf{90.8 \(\pm\) 0.1} \\
        \(P = 5\)           & 87.0 \(\pm\) 0.7            & 90.2 \(\pm\) 0.1 \\
        \(P = 10\)          & 80.5 \(\pm\) 1.2            & 89.8 \(\pm\) 0.3 \\
        \end{tabular}\label{abl:patchsizeorbit}
\end{table}


              

\begin{table}[ht]
  \centering
  \caption{\textbf{Ablating Model Size.} Performance results for varying model depth and width across three datasets: \texttt{PROTEINS}, \texttt{COX2}, and \texttt{ORBIT5K}.}
  \label{abl:model_size}
  
  \begin{tabular}{lccc}
    \\
    \toprule
    \textbf{Depth} & \textbf{\texttt{PROTEINS}}  & \textbf{\texttt{COX2}} & \textbf{\texttt{ORBIT5K}} \\ 
    \midrule
    2  & 75.0 \(\pm\) 1.9  & 83.1 \(\pm\) 3.4  & 86.2 \(\pm\) 0.9 \\
    5  & \textbf{75.7 \(\pm\) 2.7}  & 83.3 \(\pm\) 4.3  & \textbf{87.0 \(\pm\) 0.7} \\
    8  & 75.1 \(\pm\) 4.3  & \textbf{83.9 \(\pm\) 4.1}  & 86.3 \(\pm\) 0.9 \\
    \midrule
    \textbf{Width} & \textbf{\texttt{PROTEINS}} & \textbf{\texttt{COX2}} & \textbf{\texttt{ORBIT5K}} \\ 
    \midrule
    96   & 75.5 \(\pm\) 4.9  & 83.1 \(\pm\) 5.6  & 86.4 \(\pm\) 1.0 \\
    192  & \textbf{75.7 \(\pm\) 2.7}  & \textbf{83.3 \(\pm\) 4.3}  & \textbf{87.0 \(\pm\) 0.7} \\
    384  & 75.5 \(\pm\) 2.1  & 82.9 \(\pm\) 3.0  & 86.6 \(\pm\) 1.0 \\
    \bottomrule
  \end{tabular}
\end{table}

\subsection{Model Size.}
Table~\ref{abl:model_size} show the impact of model depth and width on classification performance across the \texttt{ORBIT5K} and graph datasets. The results demonstrate that xPerT performs consistently well across various configurations of depth and width, indicating its robustness in different situations. Regardless of the specific dataset or parameter setting, the model maintains competitive accuracy, suggesting that xPerT can generalize effectively across different tasks and dataset characteristics.  

\section{Conclusion and Limitations}
In this work, we introduced xPerT, a novel transformer architecture specifically designed for persistence diagrams, enabling efficient handling of topological information in data. xPerT's design allows for easy integration with other machine learning models while efficiently handling the sparse nature of persistence diagrams, significantly reducing computational complexity compared to previous transformer models for persistence diagrams.

We demonstrated the effectiveness of xPerT on both graph classification and dynamical system classification tasks, where it outperforms other methods that use (extended) persistence diagrams as machine learning features in several datasets. Furthermore, we conducted an ablation study to investigate the effects of patch size and model size on performance. Our results indicate that xPerT performs robustly across a wide range of hyperparameters and is resilient to changes in model size and grid resolution.

While xPerT shows a promise for topological data analysis, a key limitation of our study is that the model was tested on a limited number of datasets. In future work, we plan to evaluate xPerT on a more diverse set of datasets from different domains to better understand its generalizability and scalability. 

In conclusion, xPerT offers a promising approach to utilizing topological data for machine learning tasks. Its ease of integration with other models, combined with its robust performance, makes xPerT a valuable tool for a wide range of applications.

\newpage
\bibliography{sample_paper}

\newpage

\onecolumn
\appendix
\section{Proofs}

\subsection{Proof of Proposition~\ref{prop:stability}} \label{proof:stability}
We first state and prove two lemmas that will be used in the proof of Proposition~\ref{prop:stability}.
\begin{lemma}\label{lemma:projection}
Let \(D_r\) be a rotated persistence diagram. Then
\[
W(\Pi_{\delta}(D_r), D_r) \leq \sqrt{\frac{|D|}{2}} \delta.
\]
\end{lemma}
\begin{proof}
The cost of the matching \(\Pi_{\delta}: D_{r} \rightarrow \Pi_{\delta}(D_r) \) is given by 
\[
\left(\sum_{u \in D_r} \| u - \Pi_{\delta}(u) \|_2^2\right)^{1/2}  \leq \left(\sum_{u \in D_r} \frac{1}{2} \delta^2 \right)^{1/2} = \sqrt{\frac{|D|}{2}} \delta.
\]
\end{proof}

\begin{lemma}\label{lemma:stability}
   Let \( D \) and \( D' \) be two persistence diagrams. The Wasserstein distance between their rotated versions \( D_r \) and \( D'_r \) satisfies the following inequality:
   \[
   W\left( D_r, D'_r \right) \leq \sqrt{3} W(D, D').
   \]
\end{lemma}

\begin{proof}
Let \(\tilde{\gamma}\) be the optimal matching between \( D \) and \( D' \) that achieves the Wasserstein distance \( W(D, D') \). We decompose the persistence diagrams into disjoint unions:
\[
D = D_{\text{match}} \sqcup D_{\text{diag}}, \quad D' = D'_{\text{match}} \sqcup D'_{\text{diag}},
\]
where \( D_{\text{match}} \) is the set of points matched to \( D'_{\text{match}} \), and \( D_{\text{diag}} \), \( D'_{\text{diag}} \) are the points matched to the diagonal. Let \( D_{\text{match}} = \{u_1, \dots, u_k\} \), \( D'_{\text{match}} = \{v_1, \dots, v_k\} \), and assume without loss of generality that \(D_{\text{match}} \neq \varnothing\) and \( v_i = \tilde{\gamma}(u_i) \) for \( 1 \leq i \leq k \). The Wasserstein distance between \( D \) and \( D' \) is given by
\begin{align}\label{eq:1}
   W^2(D, D') = \sum_{i=1}^{k} \| u_i - v_i \|_2^2 + \sum_{i=k+1}^{k+l} \| u_i - \pi(u_i) \|_2^2 + \sum_{i=k+1}^{k+m} \| v_i - \pi(v_i) \|_2^2,
\end{align}
where \( \pi \) denotes the projection onto the diagonal. 

Now, consider the Wasserstein distance between the rotated diagrams \( D_r \) and \( D'_r \). 
Using the matching \(\gamma = R \tilde{\gamma} R^{-1}\) between \(D_r\) and \(D_{r}^{'}\), we have
\[
W^2(D_r, D'_r) \leq \sum_{i=1}^{k} \| Ru_i - Rv_i \|_2^2 + \sum_{i=k+1}^{k+l} \| Ru_i - \pi_x(Ru_i) \|_2^2 + \sum_{i=k+1}^{k+m} \| Rv_i - \pi_x(Rv_i) \|_2^2,
\]
where \( R \) is the rotation matrix and \( \pi_x \) is the projection in the birth-persistence plane.

Since rotation preserves distances, \( \| Ru_i - Rv_i \|_2 = \| u_i - v_i \|_2 \), and the Frobenius norm \( \| R \|_F \leq \sqrt{3} \) implies
\[
W^2(D_r, D'_r) \leq \sum_{i=1}^{k} 3 \| u_i - v_i \|_2^2 + 2 \sum_{i=k+1}^{k+l} \| u_i - \pi(u_i) \|_2^2 + 2 \sum_{i=k+1}^{k+m} \| v_i - \pi(v_i) \|_2^2.
\]
Thus, we obtain
\[
W^2(D_r, D'_r) \leq 3 W^2(D, D'),
\]
which gives the desired inequality
\[
W(D_r, D'_r) \leq \sqrt{3} W(D, D').
\]
The same proof holds when \(D_{\text{match}} = \varnothing\) if we use only the last two terms in the equation \ref{eq:1}.
\end{proof}

Now, the proof of Proposition~\ref{prop:stability} follows directly from Lemmas~\ref{lemma:projection} and~\ref{lemma:stability}.\\

\noindent \textbf{Proposition \ref{prop:stability}.} 
Let \( D \) and \( D' \) be two persistence diagrams, and suppose that \( \delta < W(D, D') \). Then the following inequality holds:
\[
W\Bigl( \pi_r(D), \pi_r(D') \Bigr) 
\leq \left( \frac{\sqrt{|D|}+\sqrt{|D^{'}|}}{\sqrt{2}}+ \sqrt{3} \right) W(D, D')
\]
\begin{proof}
   We begin by applying the triangle inequality in the Wasserstein metric space:
   \[
   W\left( \pi_r(D), \pi_r(D') \right) \leq W\left( \pi_r(D), D_r \right) + W\left( D_r, D_r' \right) + W\left( D_r', \pi_r(D') \right).
   \]
   The terms \( W\left( \pi_r(D), D_r \right) \) and \( W\left( D_r', \pi_r(D') \right) \) can be bounded using Lemma \ref{lemma:projection}, which gives
   \[
   W\left( \pi_r(D), D_r \right) \leq \sqrt{\frac{|D|}{2}} \delta 
   \quad \text{and} \quad 
   W\left( D_r', \pi_r(D') \right) \leq \sqrt{\frac{|D^{'}|}{2}} \delta.
   \]
   Next, using Lemma \ref{lemma:stability}, we bound the middle term:
   \[
   W\left( D_r, D_r' \right) \leq \sqrt{3} W(D, D').
   \]
   Thus, combining these inequalities, we obtain
   \[
   W\left( \pi_r(D), \pi_r(D') \right) 
   \leq \sqrt{\frac{|D|}{2}} \delta
   + \sqrt{3} W(D, D') 
   + \sqrt{\frac{|D^{'}|}{2}} \delta,
   \]
   which simplifies to
   \[
   W\left( \pi_r(D), \pi_r(D') \right) 
   \leq \left( \frac{\sqrt{|D|}+\sqrt{|D^{'}|}}{\sqrt{2}}+ \sqrt{3} \right) W(D, D').
   \]
   This completes the proof.
   \end{proof}

\section{Experimental Details}
\subsection{Comparison of Baselines in Graph Classification}
Table \ref{tab:baselines} shows the descriptions of the baseline models in graph classification in section \ref{sec:graph classification}. Note that even though the models in Table \ref{tab:graph} are based on persistent homology related inputs, the specific inputs of each model are different. 

\begin{table}[]
   \centering
   \caption{Summary of topological baselines in graph classification. MP indicates the model is based on multi-parameter persistent homology.}
   \begin{tabular}{c|c|c|c} 
   \multicolumn{1}{l}{}    & \multicolumn{1}{l}{}    & \multicolumn{1}{l}{}  & \multicolumn{1}{l}{}\\  
   Method                                                     & MP & Filtration                   & Explanation \\ \hline
   \begin{tabular}[c]{@{}l@{}}PersLay\\ ATOL\end{tabular}     &\(\times\) & \begin{tabular}[c]{@{}l@{}}- HKS\(_{0.1}\)\\ - HKS\(_{10.}\)\end{tabular} & \begin{tabular}[c]{@{}l@{}}Each diagram in the extended persistence diagrams \\ is transformed into a vector. Resulting 8 vectors in \\ total are concatenated to generate a representation.\end{tabular}\\ \hline

   \begin{tabular}[c]{@{}l@{}}Persformer\\ xPerT\end{tabular} &\(\times\) & - HKS\(_{1.0}\)  & Extended diagram is transformed to a vector.            \\ \hline

   GRIL                                                       &\(\circ\)  & \begin{tabular}[c]{@{}l@{}}- HKS\(_{1.0}\) + HKS\(_{10.}\)\\ \text{- Ricci curvature}\end{tabular} & \begin{tabular}[c]{@{}l@{}}Persistence landscape is computed based on \\the generalized rank invariant.\end{tabular} \\ \hline
   HSM-MP-SW    &\(\circ\)  & \begin{tabular}[c]{@{}l@{}}- HKS\(_{10.}\)\\ \text{- Ricci curvature}\end{tabular}                 & \begin{tabular}[c]{@{}l@{}}Sliced Wasserstein kernel is computed based on \\the signed barcode.\end{tabular}           
   \end{tabular}\label{tab:baselines}
\end{table}

\subsection{Hyperparameters} \label{sec:hyperparameters}
Table \ref{tab:hyperparameters} and \ref{tab:architecture} show the hyperparameters and the architecture of the xPerT and Persformer used in the experiments. The hyperparameters are chosen based on the performance of the models on the mean accuracy over 10-fold cross validations. The architecture of the xPerT is the same for all experiments, while the Persformer uses different architectures for graph and orbit datasets as in the original paper.

\begin{table}[]
   \centering
   \caption{Hyperparameters for the training of xPerT and Persformer on graph and orbit datasets.}
   \begin{tabular}{lllll}
                  & & & & \\
   \textbf{config}        & \textbf{\begin{tabular}[c]{@{}l@{}}xPerT\\ Graph\end{tabular}} & \textbf{\begin{tabular}[c]{@{}l@{}}xPerT\\ Orbit\end{tabular}} & \textbf{\begin{tabular}[c]{@{}l@{}}Persformer\\ Graph\end{tabular}}  & \textbf{\begin{tabular}[c]{@{}l@{}}Persformer\\ Orbit\end{tabular}} \\ \hline
   optimizer     & AdamW           & AdamW            & AdamW           & AdamW           \\
   learning rate & 1e-3            & 1e-4             & 1e-3            & 1e-3            \\
   weight decay  & 5e-2            & 5e-2             & 5e-2            & 5e-2            \\
   batch size    & 64              & 64               & 64              & 16              \\
   lr schedule   & cosine decay    & cosine decay     & cosine decay    & cosine decay    \\
   warmup epochs & 50              & 50               & 50              & 50              \\
   epochs        & 300             & 300              & 300             & 300
\end{tabular} \label{tab:hyperparameters}
\end{table}

\begin{table}[]
   \centering
   \caption{Transformer architecture for xPerT and Persformer used in the experiments. Note that the persformer uses different architectures for graph and orbit datasets as in the original paper.}
   \begin{tabular}{llll}
               & & & \\
               & \textbf{\begin{tabular}[c]{@{}l@{}}xPerT\\ \end{tabular}} & \textbf{\begin{tabular}[c]{@{}l@{}}Persformer\\ graph\end{tabular}} & \textbf{\begin{tabular}[c]{@{}l@{}}Persformer\\ orbit\end{tabular}} \\ \hline
   \# layers  & 5                                                                       & 2                                                                     & 5                                                                     \\
   \# heads   & 8                                                                       & 4                                                                     & 8                                                                     \\
   token dim. & 192                                                                     & 32                                                                    & 128                                                                  
\end{tabular} \label{tab:architecture}
\end{table}

\section{Persistent Homology} \label{sec:persistent homology}
This section provides a brief introduction to persistent homology, which is a mathematical tool for studying the topological features of a space. A \textit{filtration} is a key object in construction of persistent homology, which is an increasing sequence \((X)_{t>0}\) of subspaces of a space \(X\). The \textit{persistent homology of dimension \(k\)} is a sequence of vector spaces \(\{H_{k}(X_{t}; \sF)\}_{t>0}\) and linear maps \(\{H_{k}(X_{t};\sF) \rightarrow H_{k}(X_{t^{'}};\sF)\}_{t \leq t^{'}}\), which captures the topological features of the space as the filtration parameter \(t\) varies. Here, \(\sF\) is a field, which is usually taken to be the finite feild \(\sZ_{2}\), and \(H_{k}\) denotes the \(k\)-th homology group. Monitoring the evolution of homological features via linear maps \(\{H_{k}(X_{t};\sF) \rightarrow H_{k}(X_{t^{'}};\sF)\}_{t \leq t^{'}}\) allows associating
an interval \([b, d]\), where \(b\) and \(d\) are birth and death times, respectively. For a detailed introduction to homology and persistent homology, see \cite{computationaltopology}.

\paragraph{Filtration on Graphs}
Given a graph \(G=(V,E)\), lef \(f:V \rightarrow \sR\) be a function defined on the vertices of the graph. 
A sublevel set filtration \(\big(G_{t}\big)_{t>0}\) is an increasing sequence of subgraphs of \(G\) defined as follows. For each \(t > 0\), the sublevel set \(G_{t}\) is a subgraph of \(G\) whose vertices are the vertices in \(V\), and whose edges are the edges in \(E\) that connect the vertices whose function values are less than or equal to \(t\). Formally, 
\[
V(G_{t}) = \{v \in V \mid f(v) \leq t\}, \quad E(G_{t}) = \{(u,v) \in E \mid f(u) \leq t, f(v) \leq t\}.
\]

In this paper, we use the heat kernel signature (HKS) as the function \(f\), which is a function defined on the vertices of the graph that provides the local geometric information of the graph. The usual sublevel set filtration is used for generate the ordinary persistent homology.

For extended persistent homology, we use a filtration that combines the sublevel set and superlevel set filtrations. Given a value \(t \in \sR\), the superlevel set \(G^{t}\) of the graph \(G\) is defined as
\[
V(G^{t}) = \{v \in V \mid f(v) \geq t\}, \quad E(G^{t}) = \{(u,v) \in E \mid f(u) \geq t, f(v) \geq t\}.
\]
Without loss of generality, assume that the minimum value of \(f\) is 0. Then the extended filtration is given by 
\[
   \tilde{G}_t = \begin{cases}
      G_{t} & \text{if } t \leq M, \\\\
      G/G^{2M-t} & \text{otherwise},
   \end{cases}
\] 
where \(M\) is the maximum value of the function \(f\). Here \(G/G^{2M-t}\) is the quotient graph obtained by contracting the vertices in \(G^{2M-t}\) to a single vertex. Note that the extended filtration is not a true filtration, as it does not satisfy the monotonicity condition. Still, we can apply the homology to the sequence of graphs \(\big(\tilde{G}_t\big)_{t>0}\) to compute the extended persistence diagram, which reflects the topological features of the graph at different scales.

\paragraph{Filtration on Point Cloud}
Given a point cloud \(X = \{x_1, \dots, x_n\} \subset \sR^{d}\), the \textit{Rips filtration} \(\big(R_{t}(X)\big)_{t>0}\) is an increasing sequence \(\big(R_{t}(X)\big)_{t>0}\) of simplicial complexes defined as follows. For each \(t > 0\), the Rips complex \(R_{t}(X)\) is a simplicial complex whose vertices are the points in \(X\), and whose simplices are the subsets of \(X\) that are pairwise within distance \(t\). Formally, a simplex \(\sigma \subset X\) is in \(R_{t}(X)\) if the diameter of \(\sigma\) is less than or equal to \(t\). 

However, the computation of the Rips filtration can be quite costly even in a low-dimensional space. 
To reduce the computational cost, following \cite{persformer}\footnote{The paper mentions that the authors have used the alpha filtration, but their code uses weak alpha filtration.}, we use the \textit{weak alpha complex}, which is a Rips complex defined on the Delaunay triangulation of the point cloud.

\paragraph{Persistence Diagram}
A persistence diagram is a collection of the pairs of birth and death times of topological features obtained from
persistent homology. More concretely, the persistence diagram is a multiset of points in the extended plane \(\sR \times (\sR \cup \{\infty\})\), where each point \((b, d)\) represents a topological feature that is born at time \(b\) and dies at time \(d\). 

\section{Tokenization of Ordinary Persistence Diagrams} \label{sec:Tokenization OPD}
Given a point cloud \(X\), let \(D_{k}\) be its \(k\)-dimensional persistence diagram, computed using one of the point cloud filtrations (e.g., Rips filtration or weak alpha filtration). For machine learning input, we typically use more than one diagram, \(D = \{D_{0}, \dots, D_{k}\}\). As with the extended persistence diagram, we could stack the corresponding PPDs. However, the points in the 0-dimensional persistence diagram \(D_{0}\) typically lie only along the \(y\)-axis, while points from higher-dimensional diagrams are distributed in the 2D plane. This makes it unnatural to apply the channel-stacking tokenization method described in \ref{sec:Tokenization}. Instead, we tokenize each diagram separately and combine them into a single sequence. More concretely, let \(\{\ppd_{H}(D_{0}), \dots, \ppd_{H}(D_{k})\}\) be the PPDs of \(D_{0}, \dots, D_{k}\), respectively. Each PPD is treated as a single-channel image, and we apply the same tokenization method to each PPD as described in \ref{sec:Tokenization}. Finally, we collect all tokens from each diagram into a single sequence.

\section{Visualization of \texttt{ORBIT5K} dataset.}
Each sample in the \texttt{ORBIT5K} dataset is a point cloud generated from a dynamical system:
\[
\begin{aligned}
    x_{n+1} &= x_n + ry_n(1 - y_n) &\mod 1 \\
    y_{n+1} &= y_n + rx_{n+1}(1 - x_{n+1}) &\mod 1
\end{aligned}
\]
where \(r\) is a parameter that determines the behavior of the system. The Figure~\ref{fig:orbits} shows the examples of the orbit datasets with different values of \(r\). 
\begin{figure}[h]
   \centering
   \includegraphics[scale=0.27]{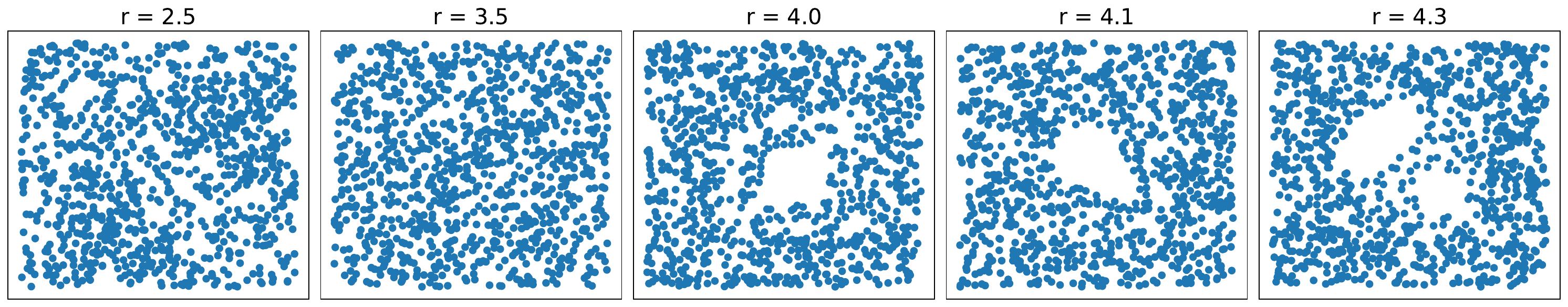}
   \caption{Examples of orbit datasets with different value of \(r\).}\label{fig:orbits}
\end{figure}

\end{document}

%% file: math_commands.tex
\usepackage{amsmath,amsfonts,bm,amsthm,amssymb}

\def\eqref#1{equation~\ref{#1}}

\def\1{\bm{1}}

\DeclareMathAlphabet{\mathsfit}{\encodingdefault}{\sfdefault}{m}{sl}
\SetMathAlphabet{\mathsfit}{bold}{\encodingdefault}{\sfdefault}{bx}{n}

\def\sF{{\mathbb{F}}}

\def\sN{{\mathbb{N}}}

\def\sR{{\mathbb{R}}}

\def\sZ{{\mathbb{Z}}}

